%% file: main.tex
\documentclass{article}
\usepackage[final,sglblindworkshop]{neurips_2026}
\workshoptitle{PALM: Personalized, Aligned, Long-Term Memory for AI Systems}
\makeatletter
\renewcommand{\@noticestring}{Accepted @ Personalized, aligned, long-term memory for AI systems Workshop, NeurIPS 2026}
\makeatother

\usepackage[utf8]{inputenc}
\usepackage[T1]{fontenc}
\usepackage[hidelinks]{hyperref}
\usepackage{url}
\usepackage{booktabs}
\usepackage{array}
\usepackage{amsmath,amssymb,amsthm}
\usepackage{microtype}
\usepackage{xcolor}
\usepackage{graphicx}
\usepackage{tikz}
\usetikzlibrary{positioning,fit,arrows.meta}
\usepackage{enumitem}
\usepackage{multirow}
\usepackage{flafter}
\usepackage{placeins}
\usepackage{needspace}

\input{generated/result_macros.tex}

\newtheorem{proposition}{Proposition}

\newcommand{\core}{\mathcal{C}}
\newcommand{\beh}{F}
\newcommand{\jc}{JointCore}

\title{Beyond Marginal Influence: Interaction-Aware Memory Selection for LLM Agents}
\author{%
  BHVSP Subrahmanyam \\
  MU-VT Interdisciplinary Advanced Research Centre for Transformative Technologies \\
  Mahindra University, India \\
  \texttt{subrahmanyam.bh@mahindrauniversity.edu.in}
  \And
  Neeraj Choudhary \\
  Department of Computer Science and Engineering \\
  Mahindra University, India \\
  \texttt{neeraj.choudhary@mahindrauniversity.edu.in}
}

\begin{document}
\maketitle

\begin{abstract}
LLM-agent memory selection often uses retrieval rankings based on record
relevance or individual influence.  Such rankings cannot identify sufficient
subsets under redundant or jointly necessary evidence.  We study per-decision memory
sufficiency: selecting a small subset that preserves full-memory behavior.  To
separate search error from model error, we introduce JointCoreBench, a 
benchmark with nine interaction families, six workflow domains, natural-language
decisions, and exact sufficient-set oracles, and \jc, a relevance/cost-ordered
group-deletion baseline that tests blocks before individual records.  On
\ExactTasks{} exact conditions, cost-ordered \jc{} recovers a minimum-cost
sufficient set in \ExactJCValid{} of cases, reduces subset evaluations from
\ExactSingletonInterventions{} to \ExactJCInterventions{}
(\ExactInterventionReduction{}), and achieves \ExactJCSavings{} active-memory
savings.  Across four 12B--70B model families and \HeadlineTasksPerModel{} tasks
per model, its unweighted mean is \HeadlineJCInterventions{} versus
\HeadlineSingletonInterventions{} model queries and \HeadlineJCMinimum{} versus
\HeadlineSingletonMinimum{} minimum-cost recovery relative to cost-aware
singleton deletion.  Both methods preserve full-memory behavior under the
deterministic primary interface.  Tests with physical deletion, UNKNOWN masks,
PENDING sentinels, counterfactual values, and alternative labels identify
interface-dependent behavior.
JointCore is an efficient benchmark baseline; the study does not establish
deployment-ready memory compression.  Code is available at
\url{https://github.com/dr-pandit-69/jointcore-neurips}.
\end{abstract}

\section{Introduction}

Long-running LLM agents collect observations, tool outputs, constraints, and
decisions.  Their memory cannot grow without limit.  LongMemEval
and MemoryAgentBench test retrieval, cross-session updates, and selective
forgetting \citep{wu2024longmemeval,hu2026memoryagentbench}.  The benchmarks
show agents remain unreliable on these tasks.  Existing methods often retrieve,
rank, or evaluate records one at a time.

Decisions can depend on several records.  Suppose two independent authorization
channels can each permit an action.  Removing one leaves the decision unchanged
because the other remains.  Other rules may require two approvals, so neither
record is useful alone.  Single-record checks can miss
the first pattern because neither channel appears necessary when removed alone,
and they can miss the second when records are scored without their partners.
We instead seek a subset that preserves the full-memory
decision.  We call this \emph{behavioral sufficiency}.  When the task rule is
known, \emph{logical sufficiency} asks whether the subset contains required
evidence.  The \emph{minimum-cost} sufficient set has the lowest registered
cost.

Sufficient input subsets and removal-based explanations ask whether an output
survives input removal \citep{covert2021removing,carter2019sis}.  In agent
memory, Causal Memory Intervention measures individual-record effects, DeMem
compresses memory around a decision, and MemAudit combines counterfactual
influence with structural signals
\citep{srivastava2026cmi,zou2026demem,tan2026memaudit}.  These studies address
related questions, but existing evaluations do not supply controlled
interactions with exact sufficient and minimum-cost sets.  Without that ground
truth, we cannot separate search errors from errors made by the model judging a
subset.

Exhaustive search can give an exact answer, but the number of subsets grows
exponentially.  Each model-mediated check also requires a model query.  We
introduce JointCore to reduce this cost.  It tests a block of records at once.
If the block can be removed without changing the decision, JointCore deletes it
and skips its smaller groups.  If the decision changes, JointCore splits the
block and continues.  It orders records by relevance and registered cost.  The
search adapts delta debugging and its use in LLM-integrated systems
\citep{zeller2002delta,zhu2026ddllm} to typed memory and decision preservation.

Real agent histories do not reveal which subset is sufficient.  Two different
subsets can produce the same observed decision, so the decision alone is not
enough to evaluate a search method.  JointCoreBench provides a reference answer
for each controlled case.  It combines natural-language workflows with
redundancy, conjunction, and other interactions.  Its evaluator enumerates
every sufficient subset and identifies the minimum-cost solutions.  These
answers come from task logic, not model behavior.  Each selection method is
compared with this common target.  We then measure behavioral and logical
sufficiency, active-memory cost, and selection cost, and separate search errors
from incorrect model judgments.

We first use the exact evaluator, then repeat the selection process with Qwen,
Gemma, Mistral and Llama as the decision models.  Llama is the
largest-capacity stress test in this cohort.  The task format and search
procedure remain fixed.  We also vary the representation of omitted records
and decision labels.  These tests examine whether the conclusions depend on
the memory or label interface.

The paper makes three contributions:
\begin{enumerate}[leftmargin=*,itemsep=2pt,parsep=0pt,topsep=3pt,partopsep=0pt]
  \item JointCoreBench, an exact benchmark for behavioral and logical memory
  sufficiency, controlled interactions, minimum cost, active-memory cost, and
  selection cost;
  \item JointCore, a relevance and cost-ordered group-deletion baseline for
  finding jointly sufficient memory with fewer selection queries; and
  \item a two-stage evaluation with exact oracles and several model families,
  together with tests of memory representations and decision labels.
\end{enumerate}

\section{Problem formulation}

Fix an environment state.  Let $V=\{1,\ldots,n\}$ index its typed memory
records, with registered costs $c_i$.  For a retained subset $S\subseteq V$,
$\beh(S)=(a_1,\ldots,a_T)$ is the resulting decision trace; $\beh(V)$ is the
full-memory reference.  We call $S$ \emph{behaviorally sufficient} when
$\beh(S)=\beh(V)$.  Its cost is $c(S)=\sum_{i\in S}c_i$, and it is
\emph{minimum-cost} if no cheaper sufficient subset exists.  In the exact
benchmark, $G(S)$ denotes logical sufficiency and $\core^\star$ contains all
minimum-cost logical solutions.

An inclusion-minimal set has no sufficient proper subset.  A 1-minimal set
becomes insufficient when any one record is removed.  Neither implies minimum
cost.  We report behavioral retention, logical sufficiency, minimum-cost
recovery, retained payload cost, and selection cost separately.  A
\emph{subset evaluation} computes $\beh(S)$: the exact stage uses the oracle,
while the model stage requires one model query.  Full-memory accuracy remains
separate because behavioral retention can preserve an incorrect decision.

Define the full-set singleton effect of record $i$ as
\begin{equation}
 \Delta_i=\mathbf 1[\beh(V\setminus\{i\})\neq\beh(V)].
\end{equation}
Under redundancy, keeping only $\{i:\Delta_i=1\}$ can discard every sufficient
record.

\begin{proposition}[Singleton effects under redundancy]
For every $k\geq2$, there exists a task with $k$ individually sufficient
records for which $\Delta_i=0$ for every record, although the empty set is
insufficient.
\end{proposition}
\begin{proof}
Let $V=R$, $|R|=k$, and
$\beh(S)=\mathbf 1[S\cap R\neq\emptyset]$.  Each singleton is sufficient.  For
$k\geq2$, removing any $i$ from $V$ leaves another record, so
$\beh(V\setminus\{i\})=\beh(V)=1$ and $\Delta_i=0$.  Since
$\beh(\emptyset)=0$, the empty set is insufficient.
\end{proof}

\section{Interaction-aware memory selection}

\paragraph{JointCore search.}
JointCore initializes $S=V$ and enqueues the full ordered record sequence as
the root group of a FIFO queue.  At each step, it dequeues a group $G$ and
compares $\beh(S\setminus G)$ with the full-memory reference $\beh(V)$.  A
match removes $G$ from $S$; a mismatch bisects and enqueues $G$ when
$|G|>1$, or retains $G$ when it is a singleton.  The procedure stops when the
queue is empty or the selection budget is exhausted.  Every evaluation uses
the same environment snapshot, and an invalid model response counts as a
mismatch.

For a deterministic evaluator, $\beh(S)=\beh(V)$ is an invariant: it holds at
initialization, and $S$ changes only after a matching evaluation.  Thus, the
returned set is behaviorally sufficient under the primary memory
representation, including when the budget is exhausted.  This guarantee does
not imply logical sufficiency, minimum cost, or robustness to another memory
representation.

\paragraph{Selection-evaluation complexity.}
One selection evaluation computes $\beh(S\setminus G)$: it is a subset
evaluation in the exact stage and one model query in the model stage.  Excluding
the shared full-memory reference evaluation, bisection defines a binary
partition tree with $n$ leaves and at most $2n-1$ nodes.  Since each node is
evaluated at most once, the registered budget $Q$ gives
\begin{equation}
  N_{\mathrm{JC}}(n,Q) \leq \min\{Q,\,2n-1\}.
\end{equation}
An accepted group deletion prunes its descendants.  The worst case visits all
$2n-1$ nodes because every non-singleton deletion changes the decision, nearly
doubling the $n$ evaluations used by conditional singleton deletion.  The
registered caps, $Q=256$ for exact evaluation and $Q=64$ for model evaluation,
exceed the corresponding tree bounds of 127 at $n=64$ and 31 at $n=16$.

\paragraph{Search order.}
The fixed development rule sets $B=4$.  It places records outside the semantic
top-$B$ in a contiguous low-relevance prefix, followed by the top-$B$ records
in decreasing payload cost.  The prefix exposes low-relevance records to block
deletion, while decreasing payload cost fixes the refinement order within the
suffix.  The ordering is deterministic and has no learned component.
Figure~\ref{fig:jointcore} traces one complete search.

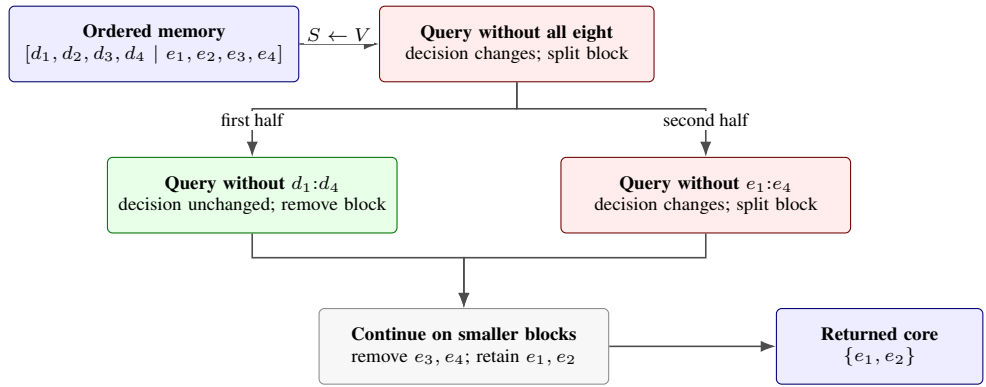
\begin{figure}[!htbp]
\centering
\begin{tikzpicture}[
  font=\scriptsize,
  >=Latex,
  flow/.style={->,semithick,draw=black!70},
  box/.style={draw=black!55,rounded corners=2pt,align=center,
              minimum height=1.0cm,inner sep=4pt},
  memory/.style={box,fill=blue!7,draw=blue!45!black},
  query/.style={box,fill=orange!9,draw=orange!55!black},
  remove/.style={box,fill=green!9,draw=green!45!black},
  split/.style={box,fill=red!7,draw=red!45!black},
  neutral/.style={box,fill=black!3,draw=black!45}
]
\node[memory,text width=3.55cm] (memory) at (1.95,2.55)
  {\textbf{Ordered memory}\\
   $[d_1,d_2,d_3,d_4\mid e_1,e_2,e_3,e_4]$};
\node[split,text width=3.35cm] (root) at (6.75,2.55)
  {\textbf{Query without all eight}\\
   decision changes; split block};
\node[remove,text width=3.55cm] (drop) at (3.25,.55)
  {\textbf{Query without $d_1{:}d_4$}\\
   decision unchanged; remove block};
\node[split,text width=3.55cm] (refine) at (9.25,.55)
  {\textbf{Query without $e_1{:}e_4$}\\
   decision changes; split block};
\node[neutral,text width=3.55cm] (continue) at (6.05,-1.45)
  {\textbf{Continue on smaller blocks}\\
   remove $e_3,e_4$; retain $e_1,e_2$};
\node[memory,text width=2.45cm] (result) at (11.55,-1.45)
  {\textbf{Returned core}\\$\{e_1,e_2\}$};
\draw[flow] (memory.east) -- node[above,fill=white,inner sep=1pt]
  {$S\leftarrow V$} (root.west);
\draw[flow] (root.south) -- ++(0,-.35) -| node[pos=.72,above,fill=white,inner sep=1pt]
  {first half} (drop.north);
\draw[flow] (root.south) -- ++(0,-.35) -| node[pos=.72,above,fill=white,inner sep=1pt]
  {second half} (refine.north);
\draw[flow] (drop.south) -- ++(0,-.32) -| (continue.north);
\draw[flow] (refine.south) -- ++(0,-.32) -| (continue.north);
\draw[flow] (continue.east) -- (result.west);
\end{tikzpicture}
\caption{Schematic JointCore trace on eight ordered records.  The search
removes $d_1{:}d_4$ as a block, refines $e_1{:}e_4$, and returns
$\{e_1,e_2\}$.}
\label{fig:jointcore}
\end{figure}

\paragraph{Comparators.}
We compare full memory; random-$B$ and semantic top-$B$; independent full-set
leave-one-out; cost-aware conditional singleton backward selection; legacy and
cost-ordered JointCore; and an exact minimum-cost oracle.  Static methods use no
selection evaluations; the oracle is an upper bound, not an implementable model
selector.  Independent leave-one-out tests every deletion against $V$, whereas
conditional singleton updates the retained set after each accepted deletion.
Cost-aware singleton is the primary efficiency comparator: each query tests one
record, whereas JointCore can test an entire block.
Both conditional methods use the same full-memory decision as their reference.

\section{Benchmark and evaluation protocol}

\subsection{Benchmark construction}

JointCoreBench instantiates nine logical families across six workflow domains:
orders, access control, deployment, calendar, database, and file debugging.
The families are singleton, redundant OR, synergistic AND, XOR, threshold,
multiple minimum sets, proxy-confounded, delayed-state, and stale-policy
conjunction.  Stale-policy is a controlled conflict pattern, not a safety
evaluation.  Support sizes are 2 and 4; record values, order, costs, semantic
scores, target-action orientation, and seed blocks are deterministic.
Appendix~\ref{app:benchmark-details} states each logical rule and follows one
task from its full memory to the selected core.

We evaluate the same task grammar in two stages.  The exact stage crosses the
nine families and six domains with candidate counts
$n\in\{8,16,32,64\}$, two support sizes, four horizon labels, and four seed
blocks, for \ExactTasks{} conditions.  Here, the logical oracle evaluates each
candidate subset without a model judgment.  The model stage fixes $n=16$ and
uses all nine families, six domains, two support sizes, and four seed/horizon
blocks, for \HeadlineTasksPerModel{} tasks per checkpoint.  ``Horizon'' is a
workflow-condition label; each trace contains one condition decision and one
terminal action.

\subsection{Primary evaluation}

The primary model interface replaces each omitted record value with a typed
UNKNOWN token, so every query preserves the original schema.  Each selected
subset is compared with the same model's full-memory decision.  We report the
outcomes defined in Section~2, together with full-memory control accuracy.  In
the exact stage, selection cost is measured in subset evaluations; in the model
stage, each such check is one model query.  Active-memory cost sums the
registered synthetic payload weights.  It excludes the fixed prompt scaffold
and is not a token-count or latency measurement.

\subsection{Model cohort and inference}

The displayed cohort comprises Qwen2.5-14B \citep{yang2025qwen25}, Gemma-4-12B-it
\citep{google2026gemma4}, Mistral-Small-3.1-24B-Instruct
\citep{mistral2025small31}, and Llama-3.3-70B-Instruct
\citep{meta2024llama33}.  Llama is an exploratory largest-capacity stress test
in this cohort.  Because the 12B display cutoff was not preregistered, the
cohort is descriptive and supports no scale claim.

Qwen and Gemma use BF16 weights.  Mistral uses a registered NF4 derivative;
Llama uses the official pinned snapshot with runtime NF4.  Both NF4 paths use
double quantization and BF16 compute \citep{dettmers2023qlora}, but are not
treated as BF16-weight evidence.  All models run locally on 40GB A100 GPUs with
Transformers 5.15, PyTorch 2.6/CUDA 12.4, a 768-token input cap, greedy mean
conditional log-probability scoring, and no sampling.

\subsection{Robustness protocol}

On a fixed 108-task subset, we reevaluate selected memories using physical
deletion, UNKNOWN masks, a schema-compatible PENDING value, and counterfactual
values.  We compare semantic YES/NO decisions with direct-action and
counterbalanced A/B interfaces; the primary condition uses UNKNOWN masks and
semantic YES/NO.  The subset contains every fourth model-stage task and covers
every family/domain/support cell, with seed block 0 and horizon label 1.

If the direct-action and A/B interfaces reverse the method ordering, the
registered fallback restores the full-memory archive and records zero payload
saving.

\Needspace{8\baselineskip}
\subsection{Development controls and statistical analysis}

Interface choices were made with separate development models.  An audit found
and corrected an iterator that overrepresented the first two interaction
families before the reported experiments; final development uses eight tasks
per family.  These corrections preceded the reported experiments.

We report each model separately and use the unweighted mean across the four
displayed models only as a descriptive summary.  The registered statistical
family contains twelve tests: behavioral retention, active-payload savings,
and model-query count for each model.  Binary retention uses exact McNemar
tests, while continuous outcomes use paired sign-flip tests.  Query-count
intervals use 2,000 task-cluster bootstrap replicates, and the associated
$p$ values are Holm-adjusted across the twelve-test family
\citep{holm1979simple}.  Appendix Table~\ref{tab:main-query-statistics} reports
these query-count tests.  Exact-sufficiency and minimum-cost comparisons are
descriptive.

\section{Results}

We begin with the exact evaluator, then examine the same search with model
judgments.  We next separate results by interaction family and test whether
they persist across memory representations, decision labels, and models.

\subsection{Exact-stage recovery and subset-evaluation cost}

Table~\ref{tab:exact} compares recovery, retained cost, and selection cost when
every candidate subset is judged by the exact oracle.

\begin{table}[!htbp]
\centering
\small
\setlength{\tabcolsep}{5pt}
\caption{Exact-stage aggregate over \ExactTasks{} conditions.  Exact
sufficiency and minimum-cost recovery are proportions.  Savings are relative
to full registered payload cost.  Higher is better except for subset
evaluations, which exclude the shared full-memory reference.}
\label{tab:exact}
\begin{tabular}{lrrrr}
\toprule
Method & Exact suff. & Min.-cost & Savings & Subset evals. \\
\midrule
Random-$B$ & 0.162 & 0.001 & 0.766 & 0.00 \\
Semantic top-$B$ & 0.944 & 0.222 & 0.775 & 0.00 \\
Independent leave-one-out & 0.667 & 0.667 & 0.900 & 30.00 \\
Cost-aware singleton & 1.000 & 1.000 & 0.878 & 30.00 \\
Legacy JointCore & 1.000 & 0.933 & 0.876 & 10.96 \\
Cost-ordered JointCore & 1.000 & 1.000 & 0.878 & 11.13 \\
Exact oracle & 1.000 & 1.000 & 0.878 & 0.00 \\
\bottomrule
\end{tabular}
\end{table}

The zero-query retrieval baselines expose a different tradeoff.  Semantic
top-$B$ is sufficient in 94.4\% of conditions, but recovers a minimum-cost set
in only 22.2\%; random-$B$ is sufficient in 16.2\%.  Cost-ordered JointCore
matches the exact oracle's sufficiency and selected cost on all \ExactTasks{}
conditions.  Against cost-aware singleton deletion, it
uses \ExactJCInterventions{} rather than \ExactSingletonInterventions{} subset
evaluations, a \ExactInterventionReduction{} reduction.  Legacy ordering
requires 10.96 evaluations but misses a minimum-cost set in 6.7\% of
conditions.  The difference in selection cost grows with memory size.  The
means are 10.12 versus 16.00 at $n=16$ and 14.14 versus 64.00 at $n=64$; at
$n=8$, they are 8.11 versus 8.00.

\FloatBarrier

\subsection{Model-mediated query cost and recovery}

Table~\ref{tab:model} reports all primary-interface metrics by model.
Figure~\ref{fig:model-effects} then isolates the two main comparisons: model
queries and minimum-cost recovery.

\begin{table}[!htbp]
\centering
\small
\setlength{\tabcolsep}{4pt}
\caption{Primary-interface evaluation over \HeadlineTasksPerModel{} tasks per
model.  Control is exact full-memory accuracy.  Remaining cells are
cost-ordered JointCore / cost-aware singleton: behavior is agreement with the
model's full-memory decision, exact is logical sufficiency, min.-cost is exact
minimum-cost recovery, and savings use registered payload cost.  Higher is
better except for queries.  $\dagger$Exploratory model result.}
\label{tab:model}
\begin{tabular}{lrrrrrr}
\toprule
Checkpoint & Control & Behavior & Exact & Min.-cost & Savings & Queries \\
\midrule
\input{generated/model_rows.tex}\\
\bottomrule
\end{tabular}
\end{table}

Both conditional methods retain full-memory behavior on every primary task.
JointCore uses fewer model queries for each model; the unweighted mean is
\HeadlineJCInterventions{} rather than \HeadlineSingletonInterventions{}
(\HeadlineReduction{} fewer).  All four paired query reductions have
Holm-adjusted $p<.001$ under the registered twelve-test family
(Appendix Table~\ref{tab:main-query-statistics}).  Mean minimum-cost recovery is
\HeadlineJCMinimum{} for JointCore and \HeadlineSingletonMinimum{} for
cost-aware singleton deletion.  Minimum-cost and exact-sufficiency comparisons
are descriptive.

\begin{figure}[!htbp]
\centering
\includegraphics[width=\linewidth]{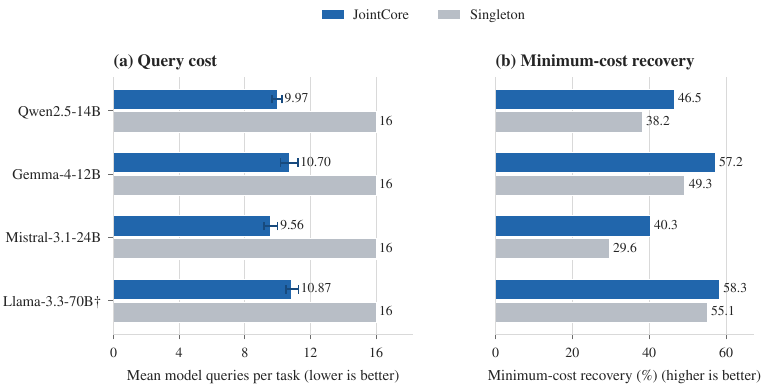}
\caption{Primary-interface comparison over \HeadlineTasksPerModel{} tasks per
model.  (a) Mean selection queries per task, where lower is better; JointCore
whiskers are 95\% task-bootstrap intervals and singleton always uses 16
queries.  (b) Exact minimum-cost recovery, where higher is better.
$\dagger$Exploratory model result.}
\label{fig:model-effects}
\end{figure}

Exact sufficiency is higher for JointCore on three rows and lower for Llama
(\LlamaJCExact{} versus \LlamaSingletonExact{}).  The model results therefore
do not show a uniform increase in logical sufficiency.

\FloatBarrier

\subsection{Results by interaction family}

Figure~\ref{fig:families} separates the model-stage differences by interaction
family.

\begin{figure}[!htbp]
\centering
\includegraphics[width=\linewidth]{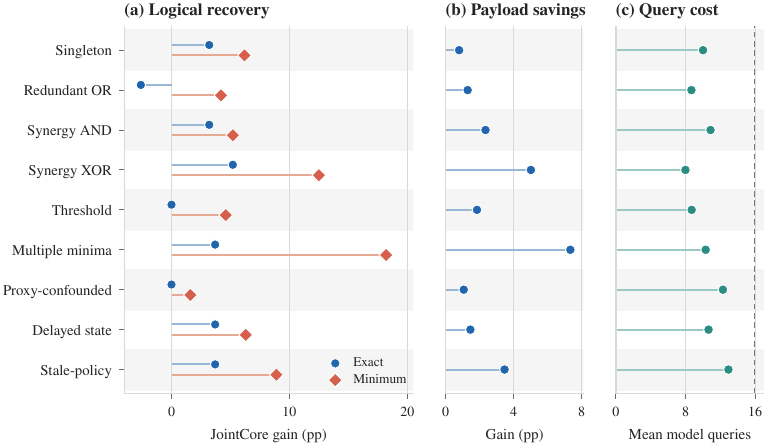}
\caption{Model-stage results by interaction family.  Panels (a) and (b) show
cost-ordered JointCore minus cost-aware singleton in percentage points for
logical recovery and active-payload savings; positive values favor JointCore.
Panel (c) shows JointCore's absolute mean query count, with singleton's
16-query cost as the reference.  Values are unweighted means across the four
displayed models and include the exploratory Llama results.}
\label{fig:families}
\end{figure}

Minimum-cost recovery differs most on multiple-minimum (+18.2 points) and XOR
(+12.5 points) tasks; these families also have the largest payload-savings
differences (+7.33 and +5.01 points).  On redundant-OR tasks, minimum-cost
recovery is 91.7\% versus 87.5\%, while exact sufficiency is 95.8\% versus
98.4\%.  The methods have equal exact sufficiency on proxy-confounded tasks
(54.2\%).  JointCore's mean query count is below the 16-query singleton
baseline in every family.

\FloatBarrier

\subsection{Memory-interface and label robustness}

Figure~\ref{fig:robustness} reports task success under each memory and label
interface in the 108-task robustness subset.

\begin{figure}[!htbp]
\centering
\includegraphics[width=\linewidth]{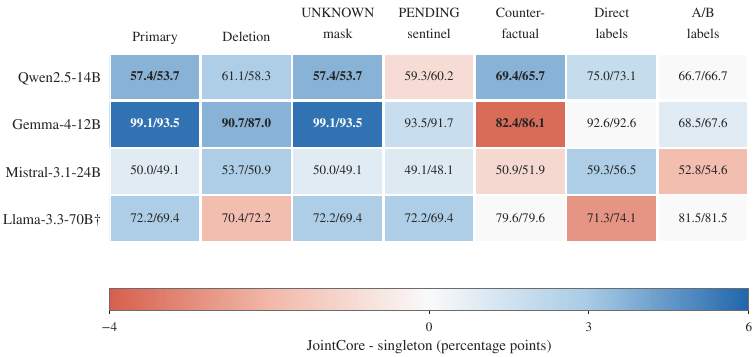}
\caption{Memory and label robustness on the fixed 108-task subset for each
model.  Each cell is cost-ordered JointCore / cost-aware singleton task success
(\%); color is their percentage-point difference.  Primary uses semantic
YES/NO decisions and UNKNOWN masks.  $\dagger$Exploratory model result.}
\label{fig:robustness}
\end{figure}

The ordering is interface-dependent.  Counterfactual values favor singleton
for Gemma (82.4\%/86.1\%), while physical deletion favors singleton for Llama
(70.4\%/72.2\%).  For Mistral, direct actions favor JointCore
(59.3\%/56.5\%), whereas A/B labels favor singleton (52.8\%/54.6\%).  For
Llama, direct actions favor singleton (71.3\%/74.1\%) and A/B labels are tied
(81.5\%/81.5\%).  Under the registered archive fallback,
\ArchiveFallbackActors{} revert to full memory, with behavior retention 1.0
and payload saving 0.0.

Together, these stress tests do not establish interface-invariant compression.

\FloatBarrier

\section{Limitations and broader impact}

\paragraph{Controlled scope.}
Exact oracles require generated workflows.  The six domain renderings share one
grammar, and the horizon labels do not represent long autonomous rollouts, so
the benchmark cannot estimate interaction prevalence in natural agent histories.
Its fixed 108-task stress subset covers every family/domain/support cell but
only one seed/horizon block.  Behavioral sufficiency can preserve a model error
and does not establish truth or causal importance.

\paragraph{Model and interface scope.}
The 12B--70B display cohort uses a non-preregistered cutoff and cannot support
a scaling claim; Llama remains exploratory.  The models also differ in
precision: Qwen and Gemma use BF16 weights, whereas Mistral and Llama use
distinct NF4 loading paths.  Finally, deletion, UNKNOWN masks, PENDING,
counterfactual values, and decision labels are not interchangeable.  The
observed reversals limit generalization beyond the tested interfaces.

\paragraph{Deployment implications.}
Registered payload weights are not token counts or latency measurements.
Although smaller active memories may reduce inference cost and disclosure,
removing authorization, consent, or provenance records can enable harmful
actions.  A deployed selector would require a trusted archive, access controls,
human review, and domain-specific conservative rules; its model queries may
themselves expose sensitive records.

\section{Conclusion}

JointCoreBench evaluates per-decision memory sufficiency under controlled
redundancy and synergy.  In exact evaluation, cost-ordered JointCore recovers a
minimum-cost sufficient set on all \ExactTasks{} conditions with
\ExactJCInterventions{} subset evaluations, compared with
\ExactSingletonInterventions{} for cost-aware singleton deletion.  Across four
models, mean query count decreases from \HeadlineSingletonInterventions{} to
\HeadlineJCInterventions{}, and mean minimum-cost recovery increases from
\HeadlineSingletonMinimum{} to \HeadlineJCMinimum{}.  Full-memory behavior is
retained under the primary interface, but memory and label changes show
interface dependence.  Evaluation on natural histories and longer trajectories
is required to assess prevalence, interface robustness, and deployment risk.

\section*{Acknowledgment}
The authors acknowledge Mahindra University for access to the DGX supercomputing facility and the technical and financial support from the Mahindra University--Virginia Tech Interdisciplinary Advanced Research Centre for Transformative Technologies (MU-VT IARCT$^{2}$), Project No. MU-VT IARCT$^{2}$/TA-AIML/2025/01.

\label{main-content-end}

\bibliographystyle{plainnat}
\bibliography{references}

\clearpage
\appendix
\section{Benchmark definitions and worked example}
\label{app:benchmark-details}

\subsection{Interaction families}

Table~\ref{tab:family-definitions} gives the exact retained-set condition for
each benchmark family.  Let $R$ denote the generated support records,
$R_A,R_B$ the two equal authorization channels, and
$k=\lceil |R|/2\rceil$.  Records outside these sets are distractors.  A listed
solution is minimum only when it also has the lowest registered payload cost.

\begin{table}[!htbp]
\centering
\footnotesize
\setlength{\tabcolsep}{4pt}
\caption{Exact interaction families used by JointCoreBench.  The condition
column specifies when a retained set $S$ is logically sufficient.}
\label{tab:family-definitions}
\begin{tabular}{>{\raggedright\arraybackslash}p{2.45cm}
                >{\raggedright\arraybackslash}p{7.0cm}
                >{\raggedright\arraybackslash}p{3.0cm}}
\toprule
Family & Sufficiency condition & Candidate minimal sets \\
\midrule
Singleton & The designated record $r_1$ is retained. & $\{r_1\}$ \\
Redundant OR & At least one record in $R$ is retained. & Any $\{r_i\}$ \\
Synergistic AND & Every record in $R$ is retained. & $R$ \\
XOR & Every XOR flag in $R$ is retained; their fixed values contain exactly one ON flag. & $R$ \\
Threshold & At least $k$ records from $R$ are retained. & Any size-$k$ subset of $R$ \\
Multiple minimum sets & Every record in $R_A$ or every record in $R_B$ is retained. & $R_A$ or $R_B$ \\
Proxy-confounded & Every record in $R$ is retained; a higher-scored proxy is not required. & $R$ \\
Delayed-state & Every state-transition record in $R$ is retained. & $R$ \\
Stale-policy conjunction & Current authorization, required support, and the stale-policy record are retained. & $R$ \\
\bottomrule
\end{tabular}
\end{table}

\subsection{Worked authorization example}

Consider the stored 16-record order task with two interchangeable authorization
channels.  Record $m_0$ validates channel A and $m_{10}$ validates channel B;
either record alone preserves the full-memory action.  The other fourteen
records are distractors.  Full memory costs 81, while either minimum sufficient
set costs 3.  Independent leave-one-out removes both authorization records
because each appears unnecessary while the other remains, and therefore returns
an insufficient empty set.

JointCore tests deletion groups of sizes 16, 8, 8, 4, 4, 2, 2, 1, and 1.
Decision-preserving tests remove the groups at evaluations 2, 4, 6, and 8;
decision-changing tests split the group or retain the final singleton.  The
procedure returns $\{m_{10}\}$ after nine subset evaluations.  Conditional
singleton deletion reaches the same minimum set after sixteen evaluations.
Table~\ref{tab:worked-example} summarizes the complete comparison.

\begin{table}[!htbp]
\centering
\small
\setlength{\tabcolsep}{5pt}
\caption{Exact worked example from the frozen benchmark.  Static retrieval
uses no selection evaluations; the exact oracle is a reference, not a selector.}
\label{tab:worked-example}
\begin{tabular}{lrrrr}
\toprule
Method & Retained memory & Cost & Subset evals. & Exact / minimum \\
\midrule
Full memory & all 16 records & 81 & 0 & yes / no \\
Semantic top-$4$ & $\{m_0,m_9,m_{10},m_{11}\}$ & 20 & 0 & yes / no \\
Independent leave-one-out & $\emptyset$ & 0 & 16 & no / no \\
Cost-aware singleton & $\{m_{10}\}$ & 3 & 16 & yes / yes \\
Cost-ordered JointCore & $\{m_{10}\}$ & 3 & 9 & yes / yes \\
Exact oracle & $\{m_0\}$ & 3 & 0 & yes / yes \\
\bottomrule
\end{tabular}
\end{table}

\clearpage
\section{Supplementary query statistics}
\label{app:main-query-statistics}

Table~\ref{tab:main-query-statistics} reports the paired query-count analysis
for the four models displayed in the main paper.

\begin{table}[ht]
\centering
\small
\setlength{\tabcolsep}{6pt}
\caption{Displayed-cohort paired model-query comparisons.  Differences are
JointCore minus cost-aware singleton, so negative values favor JointCore.
Adjusted $p$ uses the registered twelve-test Holm family.}
\label{tab:main-query-statistics}
\begin{tabular}{lrrr}
\toprule
Checkpoint & $\Delta$ queries & 95\% CI & adjusted $p$ \\
\midrule
\input{generated/statistics_rows.tex}\\
\bottomrule
\end{tabular}
\end{table}

\clearpage
\input{checklist.tex}

\end{document}

%% file: generated/result_macros.tex
\newcommand{\ExactTasks}{6,912}
\newcommand{\ExactJCValid}{100.0\%}
\newcommand{\ExactJCSavings}{87.8\%}
\newcommand{\ExactJCInterventions}{11.13}
\newcommand{\ExactSingletonInterventions}{30.00}
\newcommand{\ExactInterventionReduction}{62.9\%}

\newcommand{\HeadlineTasksPerModel}{432}

\newcommand{\HeadlineJCMinimum}{50.6\%}
\newcommand{\HeadlineSingletonMinimum}{43.1\%}

\newcommand{\HeadlineJCInterventions}{10.28}
\newcommand{\HeadlineSingletonInterventions}{16.00}
\newcommand{\HeadlineReduction}{35.8\%}

\newcommand{\LlamaJCExact}{70.6\%}
\newcommand{\LlamaSingletonExact}{71.3\%}

\newcommand{\ArchiveFallbackActors}{Mistral--3.1--24B NF4, Llama--3.3--70B NF4}

%% file: generated/model_rows.tex
Qwen2.5--14B & 0.988 & 1.000/1.000 & 0.586/0.544 & 0.465/0.382 & 0.881/0.860 & 9.97/16.00 \\
Gemma--4--12B & 0.875 & 1.000/1.000 & 0.845/0.808 & 0.572/0.493 & 0.837/0.807 & 10.70/16.00 \\
Mistral--3.1--24B NF4 & 0.847 & 1.000/1.000 & 0.660/0.644 & 0.403/0.296 & 0.874/0.835 & 9.56/16.00 \\
Llama--3.3--70B NF4$^{\dagger}$ & 0.993 & 1.000/1.000 & 0.706/0.713 & 0.583/0.551 & 0.856/0.836 & 10.87/16.00

%% file: generated/statistics_rows.tex
Qwen2.5--14B & -6.03 & [-6.34, -5.73] & $<.001$ \\
Gemma--4--12B & -5.30 & [-5.81, -4.77] & $<.001$ \\
Mistral--3.1--24B NF4 & -6.44 & [-6.84, -6.00] & $<.001$ \\
Llama--3.3--70B NF4$^{\dagger}$ & -5.13 & [-5.49, -4.74] & $<.001$

%% file: checklist.tex
\section*{NeurIPS Paper Checklist}

\begin{enumerate}

\item {\bf Claims}
    \item[] Question: Do the main claims made in the abstract and introduction accurately reflect the paper's contributions and scope?
    \item[] Answer: \answerYes{}
    \item[] Justification: The abstract and Introduction state the benchmark-conditioned claims, contributions, and scope.  Sections~5--6 report the supporting results and their limits.

\item {\bf Limitations}
    \item[] Question: Does the paper discuss the limitations of the work performed by the authors?
    \item[] Answer: \answerYes{}
    \item[] Justification: Section~6 discusses the controlled benchmark scope, model and interface dependence, quantization differences, interface limitations, and deployment risks.

\item {\bf Theory assumptions and proofs}
    \item[] Question: For each theoretical result, does the paper provide the full set of assumptions and a complete (and correct) proof?
    \item[] Answer: \answerYes{}
    \item[] Justification: Proposition~1 states its finite construction and assumptions, and its complete proof follows immediately after the statement in Section~2.

    \item {\bf Experimental result reproducibility}
    \item[] Question: Does the paper fully disclose all the information needed to reproduce the main experimental results of the paper to the extent that it affects the main claims and/or conclusions of the paper (regardless of whether the code and data are provided or not)?
    \item[] Answer: \answerYes{}
    \item[] Justification: Sections~3--4 specify the search procedure, comparators, deterministic task grids, model revisions, memory representations, scoring rule, runtime configuration, and statistical analysis.

\item {\bf Open access to data and code}
    \item[] Question: Does the paper provide open access to the data and code, with sufficient instructions to faithfully reproduce the main experimental results, as described in supplemental material?
    \item[] Answer: \answerYes{}
    \item[] Justification: Code, deterministic benchmark generation, frozen experiment configurations, tests, and reproduction commands are available at \url{https://github.com/dr-pandit-69/jointcore-neurips}.  Model weights remain available from the cited upstream publishers.

\item {\bf Experimental setting/details}
    \item[] Question: Does the paper specify all the training and test details (e.g., data splits, hyperparameters, how they were chosen, type of optimizer) necessary to understand the results?
    \item[] Answer: \answerYes{}
    \item[] Justification: Section~4 specifies the exact and model-mediated grids, development history, model cohort, precision, software and hardware, input cap, decoding, scoring, and statistical procedure.

\item {\bf Experiment statistical significance}
    \item[] Question: Does the paper report error bars suitably and correctly defined or other appropriate information about the statistical significance of the experiments?
    \item[] Answer: \answerYes{}
    \item[] Justification: Section~4 defines the registered twelve-test family, task-cluster bootstrap, paired tests, and Holm correction.  Figure~2 and Appendix Table~\ref{tab:main-query-statistics} report the model-query intervals and adjusted tests; exact-sufficiency and minimum-cost comparisons are explicitly descriptive.

\item {\bf Experiments compute resources}
    \item[] Question: For each experiment, does the paper provide sufficient information on the computer resources (type of compute workers, memory, time of execution) needed to reproduce the experiments?
    \item[] Answer: \answerNo{}
    \item[] Justification: Section~4 reports the 40GB A100 hardware, precision, software stack, and inference settings, but the current submission does not provide per-run and total execution times or storage requirements.
    
\item {\bf Code of ethics}
    \item[] Question: Does the research conducted in the paper conform, in every respect, with the NeurIPS Code of Ethics \url{https://neurips.cc/public/EthicsGuidelines}?
    \item[] Answer: \answerYes{}
    \item[] Justification: The study uses synthetic records and local inference, and involves no human-subject data or external deployment.

\item {\bf Broader impacts}
    \item[] Question: Does the paper discuss both potential positive societal impacts and negative societal impacts of the work performed?
    \item[] Answer: \answerYes{}
    \item[] Justification: Section~6 discusses possible efficiency and disclosure benefits together with privacy exposure and the risks of removing authorization, consent, provenance, or other safety-critical records.
    
\item {\bf Safeguards}
    \item[] Question: Does the paper describe safeguards that have been put in place for responsible release of data or models that have a high risk for misuse (e.g., pre-trained language models, image generators, or scraped datasets)?
    \item[] Answer: \answerNA{}
    \item[] Justification: The work releases neither a trained model nor scraped or sensitive data; the benchmark consists of synthetic workflows and evaluation code.

\item {\bf Licenses for existing assets}
    \item[] Question: Are the creators or original owners of assets (e.g., code, data, models), used in the paper, properly credited and are the license and terms of use explicitly mentioned and properly respected?
    \item[] Answer: \answerNo{}
    \item[] Justification: Section~4 cites the model sources and identifies the evaluated revisions, but the current submission does not yet include a complete inventory of license names and terms for every model and quantized derivative.

\item {\bf New assets}
    \item[] Question: Are new assets introduced in the paper well documented and is the documentation provided alongside the assets?
    \item[] Answer: \answerYes{}
    \item[] Justification: The linked repository provides the benchmark implementation, deterministic generation procedure, frozen configurations, tests, benchmark documentation, and reproduction commands.

\item {\bf Crowdsourcing and research with human subjects}
    \item[] Question: For crowdsourcing experiments and research with human subjects, does the paper include the full text of instructions given to participants and screenshots, if applicable, as well as details about compensation (if any)? 
    \item[] Answer: \answerNA{}
    \item[] Justification: The study is entirely synthetic and does not involve crowdsourcing, human participants, or human-subject data.

\item {\bf Institutional review board (IRB) approvals or equivalent for research with human subjects}
    \item[] Question: Does the paper describe potential risks incurred by study participants, whether such risks were disclosed to the subjects, and whether Institutional Review Board (IRB) approvals (or an equivalent approval/review based on the requirements of your country or institution) were obtained?
    \item[] Answer: \answerNA{}
    \item[] Justification: No human participants, crowdsourcing, or human-subject data are involved, so IRB or equivalent review is not applicable.

\item {\bf Declaration of LLM usage}
    \item[] Question: Does the paper describe the usage of LLMs if it is an important, original, or non-standard component of the core methods in this research? Note that if the LLM is used only for writing, editing, or formatting purposes and does \emph{not} impact the core methodology, scientific rigor, or originality of the research, declaration is not required.
    %this research? 
    \item[] Answer: \answerYes{}
    \item[] Justification: Sections~4--6 describe the LLM checkpoints as a core experimental component.  This submission is human-authored; AI tools were used only in an assistive capacity, and the authors reviewed and take responsibility for the final manuscript.

\end{enumerate}

%% file: main.bbl
\begin{thebibliography}{15}
\providecommand{\natexlab}[1]{#1}
\providecommand{\url}[1]{\texttt{#1}}
\expandafter\ifx\csname urlstyle\endcsname\relax
  \providecommand{\doi}[1]{doi: #1}\else
  \providecommand{\doi}{doi: \begingroup \urlstyle{rm}\Url}\fi

\bibitem[Carter et~al.(2019)Carter, Mueller, Jain, and Gifford]{carter2019sis}
Brandon Carter, Jonas Mueller, Siddhartha Jain, and David Gifford.
\newblock What made you do this? understanding black-box decisions with
  sufficient input subsets.
\newblock In \emph{Proceedings of AISTATS}, pages 567--576, 2019.

\bibitem[Covert et~al.(2021)Covert, Lundberg, and Lee]{covert2021removing}
Ian Covert, Scott~M. Lundberg, and Su-In Lee.
\newblock Explaining by removing: A unified framework for model explanation.
\newblock \emph{Journal of Machine Learning Research}, 22\penalty0
  (209):\penalty0 1--90, 2021.

\bibitem[Dettmers et~al.(2023)Dettmers, Pagnoni, Holtzman, and
  Zettlemoyer]{dettmers2023qlora}
Tim Dettmers, Artidoro Pagnoni, Ari Holtzman, and Luke Zettlemoyer.
\newblock {QLoRA}: Efficient finetuning of quantized {LLM}s.
\newblock In \emph{Advances in Neural Information Processing Systems}, 2023.

\bibitem[{Google DeepMind}(2026)]{google2026gemma4}
{Google DeepMind}.
\newblock Introducing {Gemma 4 12B}.
\newblock
  \url{https://blog.google/innovation-and-ai/technology/developers-tools/introducing-gemma-4-12B/},
  2026.

\bibitem[Holm(1979)]{holm1979simple}
Sture Holm.
\newblock A simple sequentially rejective multiple test procedure.
\newblock \emph{Scandinavian Journal of Statistics}, 6\penalty0 (2):\penalty0
  65--70, 1979.

\bibitem[Hu et~al.(2026)Hu, Wang, and McAuley]{hu2026memoryagentbench}
Yuanzhe Hu, Yu~Wang, and Julian McAuley.
\newblock Evaluating memory in {LLM} agents via incremental multi-turn
  interactions.
\newblock In \emph{International Conference on Learning Representations}, 2026.

\bibitem[{Meta}(2024)]{meta2024llama33}
{Meta}.
\newblock {Llama 3.3 70B Instruct} model card.
\newblock \url{https://huggingface.co/meta-llama/Llama-3.3-70B-Instruct}, 2024.

\bibitem[{Mistral AI}(2025)]{mistral2025small31}
{Mistral AI}.
\newblock Mistral small 3.1.
\newblock \url{https://mistral.ai/news/mistral-small-3-1/}, 2025.

\bibitem[Srivastava(2026)]{srivastava2026cmi}
Saksham~Sahai Srivastava.
\newblock Causal intervention-based memory selection for long-horizon {LLM}
  agents.
\newblock \emph{arXiv preprint arXiv:2605.17641}, 2026.

\bibitem[Tan et~al.(2026)Tan, Yao, Jin, Yu, Wang, Fan, Lu, Liu, Zhang, Ma,
  Yang, and Sun]{tan2026memaudit}
Zhewen Tan, Yilun Yao, Huiyan Jin, Wenhan Yu, Guoan Wang, Mengyuan Fan, Liang
  Lu, Feng Liu, Xiangzheng Zhang, Duohe Ma, Tong Yang, and Lin Sun.
\newblock {MemAudit}: Post-hoc auditing of poisoned agent memory via causal
  attribution and structural anomaly detection.
\newblock \emph{arXiv preprint arXiv:2605.23723}, 2026.

\bibitem[Wu et~al.(2025)Wu, Wang, Yu, Zhang, Chang, and Yu]{wu2024longmemeval}
Di~Wu, Hongwei Wang, Wenhao Yu, Yuwei Zhang, Kai-Wei Chang, and Dong Yu.
\newblock {LongMemEval}: Benchmarking chat assistants on long-term interactive
  memory.
\newblock In \emph{International Conference on Learning Representations}, 2025.

\bibitem[Yang et~al.(2025)Yang, Yang, Zhang, et~al.]{yang2025qwen25}
An~Yang, Baosong Yang, Beichen Zhang, et~al.
\newblock {Qwen2.5} technical report.
\newblock \emph{arXiv preprint arXiv:2412.15115}, 2025.

\bibitem[Zeller and Hildebrandt(2002)]{zeller2002delta}
Andreas Zeller and Ralf Hildebrandt.
\newblock Simplifying and isolating failure-inducing input.
\newblock \emph{IEEE Transactions on Software Engineering}, 28\penalty0
  (2):\penalty0 183--200, 2002.
\newblock \doi{10.1109/32.988498}.

\bibitem[Zhu et~al.(2026)Zhu, Mansur, Sch{\"a}f, Chen, Lepoint, and
  Visser]{zhu2026ddllm}
Hao-Nan Zhu, Muhammad~Numair Mansur, Martin Sch{\"a}f, Zeya Chen, Tancr{\`e}de
  Lepoint, and Willem Visser.
\newblock Delta debugging for {LLM}-integrated systems.
\newblock In \emph{2026 IEEE/ACM 48th International Conference on Software
  Engineering: Software Engineering in Practice (ICSE-SEIP)}, 2026.
\newblock \doi{10.1145/3786583.3786854}.

\bibitem[Zou et~al.(2026)Zou, Guo, Liang, Wang, Wang, Wen, King, Qu, and
  Xu]{zou2026demem}
Mingxi Zou, Zhihan Guo, Langzhang Liang, Zhuo Wang, Qifan Wang, Qingsong Wen,
  Irwin King, Lizhen Qu, and Zenglin Xu.
\newblock Remember the decision, not the description: A rate-distortion
  framework for agent memory.
\newblock \emph{arXiv preprint arXiv:2605.10870}, 2026.

\end{thebibliography}
